\documentclass[a4paper]{article}

\title{Results about Sets of Desirable Gamble Sets}
\author{Catrin Campbell-Moore\footnote{University of Bristol. Thanks to Arthur van Camp, Jasper de Bock, Arne Decadt Gert de Cooman, Kevin Blackwell and Jason Konek. This paper was primarily written in 2021, overlapping with my finishing up \citet{campbellmoore2021isipta}. There is some overlap between this paper and \citet{decooman2023filterdesirable}; the results of this paper were shown independently. There was a paper by Jasper de Bock, a very early version of \citet{debock2023desirablethings}, available at the time of writing, which also had an axiomatisation allowing the gamble sets to be infinite and a representation result in line with \cref{sec:repinf}, although he adopted axiom \ref{ax:sets:infadd} which we reject. 
Both papers also omit axiom \ref{ax:sets:dominators}, which I needed; it follows from the remaining axioms when we adopt axiom \ref{ax:sets:infadd} (\cref{thm:dom from infadd}) or when $\Omega$ is finite \citep[Proposition 13]{campbellmoore2021isipta}.
}}

\usepackage{amsmath}
\usepackage{amsfonts}
\usepackage{amssymb}
\usepackage{amsthm}
\usepackage{mathrsfs} %
\usepackage{mathtools}

\usepackage[usenames,dvipsnames,table]{xcolor}
\usepackage{wasysym}
\usepackage{xspace}	

\usepackage{stmaryrd}
\usepackage{thmtools}
\usepackage{units}

\usepackage{tikz}
\usetikzlibrary{positioning,calc,decorations.pathreplacing,decorations.markings,intersections,fit,patterns,shapes,arrows.meta,
}
\usepackage{pgfplots}
\usepackage{graphicx}

\usepackage{eurosym}
\usepackage{tabularx}
\usepackage{etoolbox} %
\usepackage{xparse}
\usepackage{multirow}
\usepackage{bm}
\usepackage{microtype} %
\usepackage{booktabs}
\usepackage{enumitem} %
\usepackage[textsize=footnotesize,obeyFinal,backgroundcolor=white,bordercolor=gray!20,linecolor=gray!20,disable]{todonotes} %

\usepackage[breaklinks,colorlinks,linkcolor=black,citecolor=black,urlcolor=gray,hypertexnames=false]{hyperref} %
\usepackage{cleveref} %

\usepackage[round,comma]{natbib}
\newcommand{\bibcommand}{
	\setlength{\bibsep}{.2pt}
	
\bibliography{mybib}
	\bibliographystyle{plainnat}
	\appendix
}

\allowdisplaybreaks[1]

\tikzset{every picture/.append style={font=\small}, every label/.style={align=center}, every node/.style={align=center}}
\pgfplotsset{compat=1.15}

	\theoremstyle{plain} \newtheorem{theorem}{Theorem}[section]
	\theoremstyle{plain} \newtheorem{desired theorem}[theorem]{Desired Theorem}
	
	\theoremstyle{definition} \newtheorem{definitionx}[theorem]{Definition}
		\newenvironment{definition}
		  {\pushQED{\qed}\definitionx}
		  {\popQED\enddefinitionx}

	\theoremstyle{definition} \crefname{assumption}{assumption}{assumptions}
	\theoremstyle{plain} 
	\theoremstyle{plain} \newtheorem{sublemma}{Sublemma}[theorem]
	\theoremstyle{plain} \newtheorem{corollary}[theorem]{Corollary}
	\theoremstyle{plain} \newtheorem{proposition}[theorem]{Proposition}
	\theoremstyle{plain} 
	\theoremstyle{plain}

	\theoremstyle{remark}\newtheorem*{proofattemptx}{Proof Attempt}
	\newenvironment{proof attempt}
	  {\pushQED{\qed}\proofattemptx}
	  {\popQED\endproofattemptx}
	  
	\theoremstyle{remark}\newtheorem*{proofideax}{Proof Idea}
	\newenvironment{proof idea}
	  {\pushQED{\qed}\proofideax}
	  {\popQED\endproofideax}

\newenvironment{proof sketch}{%
  \proof}{\endproof}
  
	\theoremstyle{definition}

	\theoremstyle{remark}
	\theoremstyle{definition} 
	\theoremstyle{plain}
	\theoremstyle{plain}
	\theoremstyle{definition}

	\theoremstyle{definition}

	\theoremstyle{definition}

	\theoremstyle{definition}\newtheorem{setupx}[theorem]{Setup}\newtheorem*{setupx*}{Setup}
	\newenvironment{setup}
	  {\pushQED{\qed}\setupx}
	  {\popQED\endsetupx}

	\renewcommand{\phi}{\varphi}
	\renewcommand{\emptyset}{\varnothing}
	\renewcommand{\leq}{\leqslant}
	\renewcommand{\geq}{\geqslant}

	\newcommand{\comment}[1]{{\leavevmode\color{red}{#1}}}

	\newcommand{\powerset}[1]{\wp(#1)}
	
	\renewcommand{\Re}{\mathbb{R}}

	\usepackage{physics}

	\newcommand\SetSymbol[1][]{\:#1\vert\:
		\mathopen{}\allowbreak}
			\providecommand{\given}{\SetSymbol}
	\DeclarePairedDelimiterX\set[1]\{\}{%
		\renewcommand\given{\SetSymbol[\delimsize]}
		#1
	}

		\DeclarePairedDelimiterX\pr[1](){%
			
			#1
		}

			\DeclarePairedDelimiterX\updatepr[1](){%
				
				#1
			}

	\DeclarePairedDelimiter\seq{\langle}{\rangle}

\DeclareDocumentCommand \enumsentence { o m } {%
	\IfNoValueTF {#1} {%
		\begin{equation*}
						\text{\parbox{.85\textwidth}{#2}}
					\end{equation*}
	}{%
	\leqnomode
	\begin{align}\tag{#1}
					\qquad&\text{\parbox{.85\textwidth}{#2}}
				\end{align}
	\reqnomode
}%
}

	\newcommand{\defemph}[1]{\emph{#1}\index{#1}}

\newenvironment{suggestion}{\leavevmode\color{blue}}{}

\renewenvironment{suggestion}{\comment}{\endcomment}

\usepackage{pifont}%

\tikzset{				
	point/.style={fill=black,circle,inner sep=1pt},
	arrow/.style={-{Latex[width=1mm]},shorten <=2pt,shorten >=2pt},
}

	\renewcommand\SetSymbol[1][]{\:#1\vert\:
		\mathopen{}\allowbreak}
			\providecommand{\given}{\SetSymbol}
	\DeclarePairedDelimiterX\Set[1]\{\}{%
		\renewcommand\given{\SetSymbol[\delimsize]}
		#1
	}

		\newcommand{\filter}{{F}}
		\newcommand{\des}{{D}}
		\newcommand{\setdes}{\mathbb{D}}
		\newcommand{\setsetdes}{\mathfrak{D}}
		\newcommand{\G}{\mathcal{G}}

		\newcommand{\K}{\mathcal{K}}
		\renewcommand{\Re}{\mathbb{R}}

		\newcommand{\weakdom}{\gneq}

		\newcommand{\indicator}[1]{I_{#1}}

\crefname{axiom}{axiom}{axioms}
\Crefname{axiom}{Axiom}{Axioms}
\crefname{lemma}{lemma}{lemmas}
\Crefname{lemma}{Lemma}{Lemmas}

\makeatletter

\providecommand{\axiomtag}{}

\newenvironment{axioms}[1][]{
	\renewcommand{\axiomtag}{#1}\enumerate}{\endenumerate}

\newcommand\axiom[1]{
	\renewcommand{\theenumi}{\ensuremath{\mathrm{\axiomtag}_{#1}}}
	\renewcommand{\labelenumi}{({\theenumi})}
	\item }

\makeatother

\usetikzlibrary{through,calc}

\newcommand{\posi}{\mathrm{posi}}

\newcommand{\desext}[1]{\posi(#1\cup\posgambles)}

\newcommand{\extension}{\mathrm{Ext}}

\usepackage{stmaryrd}

\newcommand{\zerogamble}{0}

\newcommand{\posgambles}{\G_{\weakdom \zerogamble}}
\newcommand{\strposgambles}{\G_{> \zerogamble}}

\usetikzlibrary{positioning,calc,decorations.pathreplacing,decorations.markings,intersections,fit,patterns,shapes,arrows.meta,
}
\usepackage{pgfplots}
\tikzset{
pt/.style={circle,draw=none,fill=black,inner sep=.04cm},
circ/.style={circle,draw=gray,inner sep=.04cm},
region/.style={draw=none,fill=gray,fill opacity=.3},
labellingregion/.style={fill=white, fill opacity=.5,rounded corners,text opacity=1},
edge/.style={draw=gray,thick},
origin/.style={circle,draw=gray,fill=white,inner sep=.05cm}
}

\newcommand{\calA}{\mathcal{A}}

\setlist[enumerate,1]{label=\normalfont{(\roman*)}}
\setlist[enumerate,2]{label=\normalfont{(\alph*)}}

\usepackage{verbatim}
\newenvironment{mythoughts}{\comment}{\endcomment}

\makeatletter
\newcommand{\eqnum}{\refstepcounter{equation}\hfill\textup{\tagform@{\theequation}}}
\makeatother

\begin{document}
\maketitle
\begin{abstract}
Coherent sets of desirable gamble sets is used as a model for representing an agents opinions and choice preferences under uncertainty. In this paper we provide some results about the axioms required for coherence and the natural extension of a given set of desirable gamble sets. We also show that coherent sets of desirable gamble sets can be represented by a proper filter of coherent sets of desirable gambles. 
\end{abstract}

\tableofcontents
\section{Setup - Definition of Coherence}
\subsection{Gambles}
\begin{setup}
$\Omega$ is a non-empty set. 

$\G$ is the set of all gambles, which are the bounded functions from $\Omega$ to $\Re$.

When $f(\omega)\geq g(\omega)$ for all $\omega\in\Omega$, we will say $f\geq g$.

When  $f(\omega)> g(\omega)$ for all $\omega\in\Omega$, we will say $f> g$, or that $f$ \defemph{strictly dominates} $g$.

When  $f\geq g$ and $g\not\geq f$ we say $f\weakdom g$, or that $f$ \defemph{weakly dominates} $g$. Equivalently, this is when $f(\omega)\geq g(\omega)$ for all $\omega\in\Omega$ and $f(\omega)>g(\omega)$ for some $\omega\in\Omega$.

$\G_{\geq 0}$ is the set of gambles where $f\geq 0$.

$\strposgambles$ is the set of gambles which strictly dominate $\zerogamble$. I.e., where $f> \zerogamble$.

$\posgambles$ is the set of gambles which weakly dominate $\zerogamble$. I.e., where $f\weakdom 0$.

We will also make use of the positive linear hull of a set:
 $\posi(B):=\Set{\sum_{i=1}^n\lambda_ig_i\given n\in\mathbb{N}, \lambda_i> 0, g_i\in B}$. 
\end{setup}

\subsection{Desirable gambles - usual results}

\begin{definition}\label{def:coherentD}
$\des\subseteq\G$ is coherent if:
\begin{axioms}[D]
\axiom{0}{ \label{ax:desgamb:noncontradictory} $\zerogamble\notin\des$ }
\axiom{\weakdom 0}{ \label{ax:desgamb:regularity}If $g\in\posgambles$, then $g\in\des$ }
\axiom{\lambda}{\label{ax:desgamb:scalar}If $g\in\des$ and $\lambda> 0$, then $\lambda g\in\des$}
\axiom{+}{\label{ax:desgamb:add} If $f,g\in\des$, then $f+g\in\des$
}\qedhere
\end{axioms}
\end{definition}
\begin{proposition}\label{thm:des nat extn}
For $E\subseteq\G$, 
\begin{enumerate}
\item \label{itm:natextdes:1}If $0\notin\desext{E}$ then there is some coherent $\des$ extending $E$; and the minimal such coherent $\des$ is $\desext{E}$. 
\item \label{itm:natextdes:2}
If $0\in\desext{E}$ then there is no coherent $\des$ extending $E$.
\end{enumerate}
\end{proposition}
\begin{proposition}
 $\des$ is coherent iff $\des=\desext{E}$ for some  $E\subseteq\G$, and $0\notin \des$.
\end{proposition}
\subsection{Axioms for coherence for sets of desirable gamble sets}
I will give definitions of coherence as in \citet{campbellmoore2021isipta}, which are based on those of \citet{debock2018desirability} but diverge from them in order to accommodate the fact that I am not restricting to finite sets of gambles. In \cref{sec:whenfinite} we'll show they're equivalent in the case where gamble sets must be finite. 
\begin{definition}
$\K\subseteq\powerset{\G}$ is \emph{coherent} if it satisfies
\begin{axioms}[K]
	\axiom{\emptyset}{
	$\emptyset\not\in\K$\label{ax:sets:nontrivial}
	}
	\axiom{0}{
	If $A\in\K$ then $A\setminus\Set{\zerogamble}\in\K$.\label{ax:sets:0}
	}
	\axiom{\weakdom 0}{\label{ax:sets:pos} If $g\in\posgambles$, then $\Set{g}\in\K$ }
	\axiom{\supseteq} {\label{ax:sets:supersets}If $A\in\K$ and $B\supseteq A$, then $B\in\K$ }
	\axiom{\mathrm{Dom}} { If $A\in\K$ and for each $g\in A$,
	$f_g$ is some gamble where $f_g\geq g$,
	then $\Set{f_g\given g\in A}\in\K$. \label{ax:sets:dominators}}
	\axiom{\mathrm{Add}} { If $A_1,\ldots,A_n\in\K$ and for each sequence $\seq{g_1,\ldots,g_n}\in A_1\times\ldots\times A_n$, $f_{\seq{g_i}}$ is some member of $\posi(\Set{g_1,\ldots,g_n})$, then $$\Set{f_{\seq{g_i}}\given \seq{g_1,\ldots,g_n}\in A_1\times\ldots\times A_n}\in\K$$\label{ax:sets:add}}
\end{axioms}
\end{definition}
\noindent Recall  $\posi(B):=\Set{\sum_{i=1}^n\lambda_ig_i\given n\in\mathbb{N}, \lambda_i> 0, g_i\in B}$; so axiom \ref{ax:sets:add} could also be written as:\begin{itemize}\item[]
 If $A_1,\ldots,A_n\in\K$ then $$\Set{\sum_i\lambda^{\seq{g_1,\ldots,g_n}}_ig_i\given \seq{g_1,\ldots,g_n}\in A_1\times\ldots\times A_n}\in\K$$ where  $\lambda^{\seq{g_1,\ldots,g_n}}_1,\ldots,\lambda^{\seq{g_1,\ldots,g_n}}_n$ are reals all $\geq 0$ with at least one $> 0$.
\end{itemize}
\Citet{campbellmoore2021isipta} contains some comments on the formulation of axiom \ref{ax:sets:dominators}.

\section{Natural Extension for sets of desirable gamble sets}
\subsection{The key result about coherence for sets of desirable gamble sets}
This section gives a result which turns out to be the key important property of coherent sets of desirable gamble sets. We will afterwards show that this is is a characteristic result of coherence, i.e., it gives us the natural extension.%

\begin{theorem}\label{thm:coherentclosure}
Suppose $\K$ is coherent and $A_1,\ldots,A_n\in\K$.

If $B$ is such that for each sequence $\seq{g_1,\ldots,g_n}\in  A_1\times\ldots\times A_n$, either  $0\in\desext{\Set{g_1,\ldots,g_n}}$ or there is some $f_{\seq{g_1,\ldots,g_n}}\in B$ with $f_{\seq{g_1,\ldots,g_n}}\in\desext{\Set{g_1,\ldots,g_n}}$;\\
Then $B\in\K$. 
\end{theorem}
\begin{proof}
We start with a (not terribly interesting) lemma:
\begin{sublemma}\label{thm:coherentclosure:lem}
For $f\in \desext{\Set{g_1,\ldots,g_n}}$, whenever  $f\notin\posgambles$, there is some $h\in\posi(\Set{g_1,\ldots,g_n})$ with $f\geq h$.
\end{sublemma}
\begin{proof}
$f\in \desext{\Set{g_1,\ldots,g_n}}$ so $f=\sum_i\lambda_ig_i+\sum_j\mu_j p_j$ for some $p_j\in \posgambles$ and $\lambda_i,\mu_j\geq 0$, with at least one $>0$. 

Suppose $f\notin\posgambles$. Then some $\lambda_i>0$ because otherwise $f=\sum_j\mu_j p_j$ with each $p_j\in\posgambles$, and so we'd also have that $f\in\posgambles$. 

So, let $h=\sum_i\lambda_i g_i$, and we then know that $h\in\posi(\set{g_1,\ldots,g_n})$. Since $f=h+\sum_j\mu_j p_j$ with each $p_j\in\posgambles$, we know that $f\geq h$.
\end{proof}

Assume we have $\K$, $A_1,\ldots,A_n$ and $B$ satisfying the assumptions of the theorem, and we want to show that $B\in\K$.

We know that for any sequence $\seq{g_i}\in \bigtimes_i A_i$ with $0\notin \desext{\Set{g_1,\ldots,g_n}}$, there is some $f_{\seq{g_i}}\in \desext{\Set{g_1,\ldots,g_n}}$ with $f_{\seq{g_i}}\in B$.
When $0\in \desext{\Set{g_1,\ldots,g_n}}$, we can let $f_{\seq{g_i}}$ denote $0$; so we have that for all $\seq{g_i}\in\bigtimes_i A_i$, $f_{\seq{g_i}}\in B\cup\set{0}$ with $f_{\seq{g_i}}\in \desext{\set{g_1,\ldots,g_n}}$. 

If there is some $f\in\posgambles$ with $f\in B$ then by axioms axiom \ref{ax:sets:pos,ax:sets:supersets}, $B\in\K$. 
So we can assume that each $f_{\seq{g_i}}\notin\posgambles$ (noting that $0\notin\posgambles$), so by \cref{thm:coherentclosure:lem}, we can find $h_{\seq{g_i}}\in\posi(\set{g_1,\ldots,g_n})$ with $f_{\seq{g_i}}\geq h_{\seq{g_i}}$.
By axiom \ref{ax:sets:add}, 
\begin{equation}
\Set{h_{\seq{g_i}}\given \seq{g_i}\in \bigtimes_i A_i}\in\K.
\end{equation}
 and then by axiom \ref{ax:sets:dominators}, \begin{equation}
\Set{f_{\seq{g_i}}\given \seq{g_i}\in \bigtimes_i A_i}\in\K.
\end{equation} Since each $f_{\seq{g_i}}\in B\cup\set{0}$, by axiom \ref{ax:sets:supersets}, $B\cup\set{0}\in\K$. And thus $B\in\K$ by axiom \ref{ax:sets:0}. 
\end{proof}
Note that we have restricted this to finitely many $A_1,\ldots,A_n$ in $\K$. This is because axiom \ref{ax:sets:add} is restricted to finitely many members.

\subsection{Understanding this}

This will in fact give us maximal information about coherence. 

But first we give an idea of why this might be the right thing to do and to help understand the criterion.

Suppose $A_1,A_2\in\K$. We can then consider what else should be in $\K$.

You think that $A_1$ contains some desirable gamble, and so does $A_2$. But you leave it open which member of $A_1$ is desirable and which member of $A_2$. 

We can go through each member of $A_1$ and consider what are the consequences of that gamble being the one in virtue of which you think $A_1$ contains a desirable gamble. Similarly for $A_2$. Consider the possibility that $A_1$ is desirable in virtue of $g_1$, and $A_2$ is desirable in virtue of $g_2$. In that case, also $g_1+g_2$ is desirable, and more generally any $f\in\desext{\set{g_1,g_2}}$ is desirable; so any $B$ which contains some member of $\desext{\set{g_1,g_2}}$, i.e., when $\emptyset\neq B\cap\desext{\set{g_1,g_2}}$, contains some desirable gamble. 

\begin{figure}[h]
\begin{center}

\begin{tikzpicture}[scale=1] %

\coordinate (n) at (0,2);
\coordinate (e) at (2,0);
\coordinate (s) at (0,-2);
\coordinate (w) at (-2,0);

\coordinate (ne) at ($(n)+(e)$);
\coordinate (se) at ($(s)+(e)$);
\coordinate (nw) at ($(n)+(w)$);
\coordinate (sw) at ($(s)+(w)$);

\coordinate (o) at (0,0);

\coordinate (f) at (-1.2,1.4); 
\coordinate (g) at (1,-.25);

\begin{scope}
\clip (sw) rectangle (ne);
\draw [region] ($2*(f)$) -- (ne) -- ($3*(g)$) -- (o);
\draw [edge] (o) -- ($2*(f)$);
\draw [edge] (o) -- ($2*(g)$);
\node [origin] at (o) {};
\end{scope}
\node [labellingregion,anchor=west] at (1,1) {$\desext{\set{g_1,g_2}}$};

\draw (w) -- (e);
\draw (s) -- (n);

\node [pt] at (f) [label=below:{$g_1$}]{};
\node [pt] at (g) [label=below:{$g_2$}]{};

\node [pt,blue] at ($(f)+(g)$) [label={[text=blue]90:{$g_1+g_2$}}]{};

\end{tikzpicture}
\end{center}
\end{figure}

But this is all conditional on $g_1$ and $g_2$ being the desirable members. If $B$ is a set such that \emph{whichever} members of $A_1$ and $A_2$ are considered to be the relevant ones in virtue of which $A_i$ desirable, we can have that some member of $B$ is desirable, then we can outright conclude that $B$ contains a desirable gamble. We needn't be able to point to any particular member of $B$ which we evaluate as desirable. But whenever a member of $A_1$ and a member of $A_2$ are desirable, then so is a member of $B$.

This motivates the general constraint:
\begin{itemize}
\item []\phantomsection\label{property:1}
If for each pair $\seq{g_1,g_2}\in A_1\times A_2$, there is some  $f_{\seq{g_1,g_2}}\in\desext{\Set{g_1,g_2}}$ with $f_{\seq{g_1,g_2}}\in B$. 
Then $B\in\K$. \hfill{($*$)}
\end{itemize}

To see how this works in practice, consider \cref{fig:using}.

\newcommand{\createcanvas}
{
	\coordinate (n) at (0,2);
	\coordinate (e) at (2,0);
	\coordinate (s) at (0,-2);
	\coordinate (w) at (-2,0);

		\coordinate (ne) at ($(n)+(e)$);
		\coordinate (se) at ($(s)+(e)$);
		\coordinate (nw) at ($(n)+(w)$);
		\coordinate (sw) at ($(s)+(w)$);
		
		\coordinate (o) at (0,0); 

}

\makeatletter
\tikzset{
    reuse path/.code={\pgfsyssoftpath@setcurrentpath{#1}}
}
\makeatother
\tikzset{even odd clip/.code={\pgfseteorule}}

\newcommand{\cutoutorigin}{			
	\path (o) node[save path=\originpath,circle,inner sep=2pt](origin){}; %
	\clip[overlay,even odd clip,reuse path=\originpath] (ne) rectangle (sw); 
}

	\begin{figure}[h]
	\begin{center}
	\includegraphics{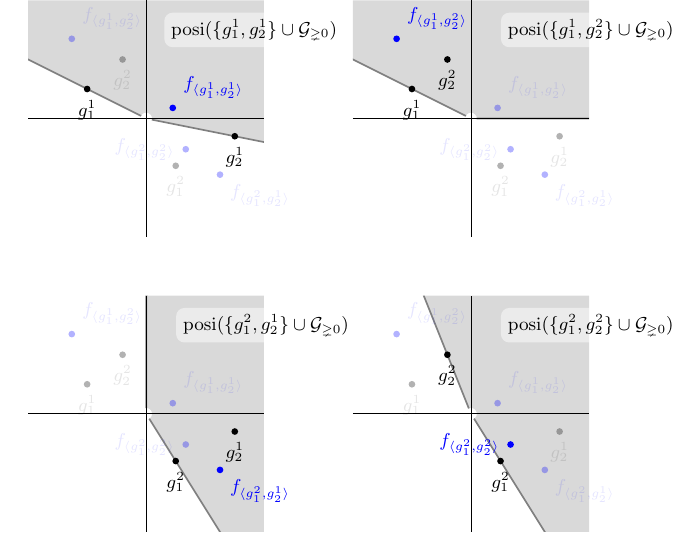}
	\end{center}
	\caption{\label{fig:using}
	Suppose $A_1=\set{g_1^1,g_1^2}\in\K$ and $A_2=\set{g_2^1,g_2^2}\in\K$. Then this reasoning  (described in \hyperref[property:1]{($*$)}) gets us that $\set{f_{\seq{g_1^1,g_2^1}},f_{\seq{g_1^1,g_2^2}},f_{\seq{g_1^2,g_2^1}},f_{\seq{g_1^2,g_2^2}}}\in\K$.}
	\end{figure}

A proof of the legitimacy of this, given the axioms on coherence, is in \cref{thm:coherentclosure} since it is a weaker condition than that used in that result.
	
This story can easily be extended to multiple sets, if $A_1,\ldots,A_n$ are all taken to contain at least one desirable gamble, and whenever we pick some particular members of them and consider those to be desirable, a member of $B$ is also desirable, then we conclude that $B$ contains a desirable gamble. 
That is:
\begin{itemize}
\item []
If for each sequence $\seq{g_i}\in \bigtimes A_i$, there is some  $f_{\seq{g_i}}\in\desext{\Set{g_1,\ldots,g_n}}$ with $f_{\seq{g_i}}\in B$. 
Then $B\in\K$. 
\end{itemize}

Can it be extended to infinitely many sets? Well, this will go beyond the axioms that have been imposed for coherence. A thought behind this: I am unable to perform infinite reasoning with desirability. Or alternatively, maybe I am unable to simultaneously consider the infinitely many sets and pick the relevant members. Or perhaps one can only make finitely many conditional assumptions.

So, we have motivated something quite close to that of  \cref{thm:coherentclosure}. 
However, it misses the clause that says that when $0\in\desext{\set{g_1,\ldots,g_n}}$ we don't need to find a $f_{\seq{g_i}}\in B$. 
This is because when $0\in\desext{\set{g_1,\ldots,g_n}}$, $g_1,\ldots,g_n$ are mutually incompatible. Suppose we have $g_1\in A_1$ and $g_2\in A_2$ with $0\in\desext{\set{g_1,g_2}}$. It cannot be that, simultaneously, $A_1$ is desirable in virtue of $g_1$ and $A_2$ is desirable in virtue of $g_2$. The desirability of $g_1$ rules out the desirability of $g_2$. So we can ignore this choice of $\seq{g_1,g_2}$. We only need to find the relevant $f_{\seq{g_i}}$ when the choice of these $g_i$ are mutually compatible.  

\begin{mythoughts}
Consider: 
\begin{align}
A_1&=\set{g_1,0}\\
A_2&=\set{g_2,0}
\end{align}
The only choice of a member of $A_1$ and a member of $A_2$ which do not result in a desirable gamble set containing $0$ is the choice of $g_1$ from $A_1$ and $g_2$ from $A_2$. So we will be able to conclude that $\set{g_1+g_2}\in\K$ since for the only coherent choice of a member of each of the sets has $g_1+g_2$ in the relevant posi, which is $\desext{\set{g_1,g_2}}$. Without this additional clause we would only be able to conclude that, for example, $\set{0,g_1+g_2}\in\K$. To be able to remove the $0$ from this set we need our additional clause that allows us to ignore mutually incompatible choices.

Let's see a further example where the gambles are really mutually incompatible rather than simply being $0$ themselves. 
Consider:
\begin{align}
A&=\set{a^1,a^2}\in\K\\
C&=\set{c^1,c^2}\in\K
\end{align}given as in:
\begin{figure}[h]
\begin{center}
\begin{tabular}{cc}
\begin{tikzpicture}[scale=.5] %

\coordinate (n) at (0,2);
\coordinate (e) at (2,0);
\coordinate (s) at (0,-2);
\coordinate (w) at (-2,0);

\coordinate (ne) at ($(n)+(e)$);
\coordinate (se) at ($(s)+(e)$);
\coordinate (nw) at ($(n)+(w)$);
\coordinate (sw) at ($(s)+(w)$);

\coordinate (o) at (0,0);

\coordinate (a) at (-1.7,.8); 
\coordinate (c) at (1,-.15);
\coordinate (b) at (-.65,1);
\coordinate (d) at (1,-1.1);

\begin{scope}
\clip (sw) rectangle (ne);
\draw [region] ($10*(a)$) -- (ne) -- ($10*(c)$) -- (o);
\draw [edge] (o) -- ($10*(a)$);
\draw [edge] (o) -- ($10*(c)$);
\node [origin] at (o) {};
\end{scope}
\node [labellingregion,anchor=west] at (1,1) {$\desext{\set{a^1,c^1}}$};

\draw (w) -- (e);
\draw (s) -- (n);

\node [pt] at ($1*(a)$) [label=below:{$a^1$}]{};
\node [circ] at ($1*(b)$) [label=above:blue{$a^2$}]{};
\node [pt] at (c) [label=below:{$c^1$}]{};

\end{tikzpicture}
&
\begin{tikzpicture}[scale=.5] %

\coordinate (n) at (0,2);
\coordinate (e) at (2,0);
\coordinate (s) at (0,-2);
\coordinate (w) at (-2,0);

\coordinate (ne) at ($(n)+(e)$);
\coordinate (se) at ($(s)+(e)$);
\coordinate (nw) at ($(n)+(w)$);
\coordinate (sw) at ($(s)+(w)$);

\coordinate (o) at (0,0);

\coordinate (a) at (-1.7,.8); 
\coordinate (c) at (1,-.15);
\coordinate (b) at (-.65,1);
\coordinate (d) at (1,-1.1);

\draw [region] (sw) rectangle (ne);
\node [labellingregion,anchor=west] at (1,1) {$\desext{\set{a^1,c^2}}$};

\draw (w) -- (e);
\draw (s) -- (n);

\node [pt] at (d) [label=below:{$c^2$}]{};
\node [pt] at (a) [label=below:{$a^1$}]{};

\end{tikzpicture}
\\
\begin{tikzpicture}[scale=.5] %

\coordinate (n) at (0,2);
\coordinate (e) at (2,0);
\coordinate (s) at (0,-2);
\coordinate (w) at (-2,0);

\coordinate (ne) at ($(n)+(e)$);
\coordinate (se) at ($(s)+(e)$);
\coordinate (nw) at ($(n)+(w)$);
\coordinate (sw) at ($(s)+(w)$);

\coordinate (o) at (0,0);

\coordinate (a) at (-1.7,.8); 
\coordinate (c) at (1,-.15);
\coordinate (b) at (-.65,1);
\coordinate (d) at (1,-1.1);

\begin{scope}
\clip (sw) rectangle (ne);
\draw [region] ($10*(c)$) -- (ne) -- ($10*(b)$) -- (o);
\draw [edge] (o) -- ($10*(c)$);
\draw [edge] (o) -- ($10*(b)$);
\node [origin] at (o) {};
\end{scope}
\node [labellingregion,anchor=west] at (1,1) {$\desext{\set{a^2,c^1}}$};

\draw (w) -- (e);
\draw (s) -- (n);

\node [pt] at ($(c)$) [label=below:{$c^1$}]{};
\node [pt] at (b) [label=below:{$a^2$}]{};

\end{tikzpicture}
&
\begin{tikzpicture}[scale=.5] %

\coordinate (n) at (0,2);
\coordinate (e) at (2,0);
\coordinate (s) at (0,-2);
\coordinate (w) at (-2,0);

\coordinate (ne) at ($(n)+(e)$);
\coordinate (se) at ($(s)+(e)$);
\coordinate (nw) at ($(n)+(w)$);
\coordinate (sw) at ($(s)+(w)$);

\coordinate (o) at (0,0);

\coordinate (a) at (-1.7,.8); 
\coordinate (c) at (1,-.15);
\coordinate (b) at (-.65,1);
\coordinate (d) at (1,-1.1);

\begin{scope}
\clip (sw) rectangle (ne);
\draw [region] ($10*(d)$) -- (ne) -- ($10*(b)$) -- (o);
\draw [edge] (o) -- ($10*(d)$);
\draw [edge] (o) -- ($10*(b)$);
\node [origin] at (o) {};
\end{scope}
\node [labellingregion,anchor=west] at (1,1) {$\desext{\set{a^2,c^2}}$};

\draw (w) -- (e);
\draw (s) -- (n);

\node [pt] at (d) [label=below:{$c^2$}]{};
\node [pt] at (b) [label=below:{$a^2$}]{};

\end{tikzpicture}
\end{tabular}
\end{center}
\end{figure}

Since here $0\in\desext{\set{a^1,c^2}}$, we do not need to consider this case. Instead we just need to ensure that there is some member of $B$ in each of $\desext{\set{a^1,c^1}}$, $\desext{\set{a^2,c^1}}$ and $\desext{\set{a^2,c^2}}$. In fact, in this case it wouldn't be an issue to ensure that we have some member of $B$ in $\desext{\set{a^1,c^2}}$ because this contains all gambles whatsoever. But not all cases are like this. %

This then fully characterises coherence. 

\end{mythoughts}

\subsection{Defining the Natural Extension}
If we start with a given set of desirable gamble sets, $\calA$, we define $\extension(\calA)$ by closing it under the construction of \cref{thm:coherentclosure} (except when we start with $\emptyset$ in which case we give a separate definition because that construction then doesn't add anything).

\begin{definition}\label{def:natext}
For $\calA\subseteq\powerset{\G}$, let $\extension(\calA)$ be given by:\\
When $\calA\neq\emptyset$: \begin{itemize}
\item $B\in\extension(\calA)$ iff there are some $A_1,\ldots,A_n\in\calA$ such that for each sequence $\seq{g_1,\ldots,g_n}\in  A_1\times\ldots\times A_n$, whenever $0\notin \desext{\Set{g_1,\ldots,g_n}}$, there is some $f_{\seq{g_1,\ldots,g_n}}\in B$ where also $f_{\seq{g_1,\ldots,g_n}}\in\desext{\Set{g_1,\ldots,g_n}}$. 
\end{itemize}
When $\calA=\emptyset$,\begin{itemize}\item  $B\in\extension(\emptyset)$ iff there is some $f\in B$ where also $f\in\posgambles$. \qedhere
\end{itemize}
\end{definition}

This will give us the natural extension notion: the minimal coherent set of desirable gamble sets extending the given set, $\calA$. The formulation is a bit different to that of \citet{debock2018desirability}, but they're ultimately equivalent. We'll discuss this in \cref{sec:relntodebock}.

We now see that this gives us the natural extension:

\begin{theorem}\label{thm:natext}
\
\begin{enumerate}
\item \label{itm:natext:1}If $\emptyset\notin\extension(\calA)$ then $\extension(\calA)$ is coherent and it is the minimal coherent $\K$ extending $\calA$. 
\item \label{itm:natext:2}
If $\emptyset\in\extension(\calA)$ then $\extension(\calA)$ is incoherent and there is no coherent $\K$ extending $\calA$. 
\end{enumerate}
\end{theorem}

\begin{proof}
We first deal with the easy case, when $\calA=\emptyset$: 
If $\K$ is coherent then by \ref{ax:sets:pos} and \ref{ax:sets:supersets}, any $B\in\extension(\emptyset)$ has $B\in\K$. Also observe that $\extension(\emptyset)$ is coherent. So it is the minimal coherent extension of $\emptyset$. 

So we now just consider when $\calA\neq\emptyset$.
As a corollary of \cref{thm:coherentclosure}, when $\K$ is coherent extending $\calA$, then $\K\supseteq\extension(\calA)$. This is because $\extension(\calA)$ was just defined to take the closure in accordance with \cref{thm:coherentclosure}.
We then need to check that $\extension(\calA)$ is coherent (so long as it doesn't contain $\emptyset$).
\begin{sublemma}Suppose $\calA\neq\emptyset$. 
If 
$\emptyset\notin\extension(\calA)$, then $\extension(\calA)$ is coherent.
\end{sublemma} \begin{proof}

We check each of the axioms.

\begin{itemize}
\item \ref{ax:sets:nontrivial}: $\emptyset\notin\extension(\calA)$ by assumption. 
\item \ref{ax:sets:0}: Suppose $A\in \extension(\calA)$. We need to show that $A\setminus\set{0}\in\extension(\calA)$. Since $A\in\extension(\calA)$, there are $B_1,\ldots,B_n\in\calA$ where for each $\seq{e_i}\in\bigtimes B_i$, whenever $0\notin \desext{\set{e_1,\ldots,e_n}}$ there is some $h_{\seq{e_i}}\in A$ with $h_{\seq{e_i}}\in \desext{\set{e_1,\ldots,e_n}}$. For such $\seq{e_i}$, $h_{\seq{e_i}}\neq 0$ so also $h_{\seq{e_i}}\in (A\setminus\set{0})$. Thus, $A\setminus\set{0}\in\extension(\calA)$. 
\item \ref{ax:sets:pos}: Let $g\in\posgambles$. Since $\calA\neq\emptyset$ there is some $B\in\calA$. For every $e\in B$, $g\in\desext{\set{e}}$. So by definition of $\extension$, $\set{g}\in \extension(\calA)$. 
\item \ref{ax:sets:supersets}: Suppose $A\in \extension(\calA)$ and $C\supseteq A$. We need to show that $C\in\extension(\calA)$.  There are some $B_i\in\calA$ where for each $\seq{e_i}\in\bigtimes B_i$, whenever $0\notin \desext{\set{e_1,\ldots,e_n}}$, we have some $h_{\seq{e_i}}\in A$ with $h_{\seq{e_i}}\in  \desext{\set{e_1,\ldots,e_n}}$. Since $C\supseteq A$, also $h_{\seq{e_i}}\in C$. So $C\in\extension(\calA)$. 
\item \ref{ax:sets:dominators}: Suppose $A\in \extension(\calA)$ and for each $g\in A$,  $f_g$ is a gamble with $f_g\geq g$. We need to show that $C:=\set{f_g\given g\in A}\in\extension(\calA)$. There are some $B_i\in\calA$ where for each $\seq{e_i}\in\bigtimes B_i$, whenever $0\notin \desext{\set{e_1,\ldots,e_n}}$, there is some $h_{\seq{e_i}}\in A$ with $h_{\seq{e_i}}\in \desext{\set{e_1,\ldots,e_n}}$. For each such $\seq{e_i}$, let $f_{\seq{e_i}}$ be the relevant $f_{h_{\seq{e_i}}}\geq h_{\seq{e_i}}$. So since $h_{\seq{e_i}}\in\desext{\set{e_1,\ldots,e_n}}$ with $f_{\seq{e_i}}\geq h_{\seq{e_i}}$, also $f_{\seq{e_i}}\in\desext{\set{e_1,\ldots,e_n}}$. And $f_{\seq{e_i}}\in C$; as required.

\item \ref{ax:sets:add}: Suppose $A_1,\ldots,A_n\in\extension(\calA)$, and $C:=\set{f_{\seq{g_i}}\given \seq{g_i}\in\bigtimes A_i}$, with $f_{\seq{g_i}}\in\posi(\set{g_1,\ldots,g_n})$ (as in \ref{ax:sets:add}). We need to show that $C\in\extension(\calA)$. 

For each $i$, there are $B^i_1,\ldots,B^i_{m_i}$ where  for any $\seq{e_j}\in \bigtimes_j B^i_j$, whenever $0\notin\desext{\set{e_1,\ldots,e_{m_i}}}$, there is some $g_{\seq{e_j}}\in A_i$ and $g_{\seq{e_j}}\in\posi(\set{e_1,\ldots, e_{m_i}}\cup \posgambles)$. 

Now, consider all of $B^i_j$ for $i,j$. This is a finite collection. Any sequence of members of it has the form $\seq{e^i_j}_{i,j}\in\bigtimes_{i,j} B^i_j$. Fix such a sequence with $0\notin \desext{\set{e^i_j\given i,j}}$. We need to show there is some $f\in C$ with $f\in \desext{\set{e^i_j\given i,j}}$.

We also have $0\notin \desext{\set{e^i_j\given j}}$ for each $i$; so we have some $g_i:=g_{\seq{e^i_j}_j}\in A_i$ with $g_i\in \desext{\set{e^i_j\given j}}$. Since $\seq{g_i}\in\bigtimes A_i$, by our choice of $C$ there is some $f\in C$ with $f\in\posi(\set{g_1,\ldots,g_n})$.

We need to show that $f\in\desext{\set{e^i_j\given i,j}}$. This follows from $f\in\posi(\set{g_1,\ldots,g_n})$ and each $g_i\in\desext{\set{e^i_j\given i,j}}$ (since it is in $\desext{\set{e^i_j\given j}}$).

So $C\in\extension(\calA)$. \qedhere
\end{itemize}

\end{proof}

So we have shown that when $\emptyset\notin\extension(\calA)$, $\extension(\calA)$ is the minimal coherent extension of $\calA$.

We now just need to check that when $\emptyset\in\extension(\calA)$, there is no coherent $\K$ extending $\calA$. This follows from \cref{thm:coherentclosure} because then $\emptyset$ is in any coherent $\K$ extending $\calA$, contradicting the supposed coherence of $\K$.
\end{proof}

\subsection{Comments on the Formulation of it}
\subsubsection{Quick notes}
One component of our natural extension formulation is: \begin{quote}
whenever $0\notin \desext{\Set{g_1,\ldots,g_n}}$, there is some $f_{\seq{g_1,\ldots,g_n}}\in B$ where also $f_{\seq{g_1,\ldots,g_n}}\in\desext{\Set{g_1,\ldots,g_n}}$. 
\end{quote}
This can equivalently be phrased:
\begin{quote}
there is some $f_{\seq{g_1,\ldots,g_n}}\in B\cup\set{0}$ where also $f_{\seq{g_1,\ldots,g_n}}\in\desext{\Set{g_1,\ldots,g_n}}$. 
\end{quote}
We can also phrase things with intersections. I.e., 
\begin{quote}
there is some $f_{\seq{g_1,\ldots,g_n}}\in (B\cup\set{0})\cap\desext{\Set{g_1,\ldots,g_n}}$. 
\end{quote}
Or without the reference to $f_{\seq{g_1,\ldots,g_n}}$ at all:
\begin{quote}
$ (B\cup\set{0})\cap\desext{\Set{g_1,\ldots,g_n}}\neq\emptyset$. 
\end{quote}

Just to see what this would look like, we'd then have the definition of $\extension$ for $\calA\neq\emptyset$ as:
\begin{itemize}
\item $B\in\extension(\calA)$ iff there are some $A_1,\ldots,A_n\in\calA$ such that for each sequence $\seq{g_1,\ldots,g_n}\in  A_1\times\ldots\times A_n$, such that $ (B\cup\set{0})\cap\desext{\Set{g_1,\ldots,g_n}}\neq\emptyset.$
\end{itemize}
And we might then want to write the case for $\calA=\emptyset$ as: 
\begin{itemize}
\item $B\in\extension(\emptyset)$ iff $B\cap \posgambles\neq\emptyset.$
\end{itemize}

I quite like the reference to the particular member, rather than talking about the intersection being non-empty, because I think it helps thinking about it and is useful to have for proofs, which is why I've opted for that. 

\subsubsection{On the case when $\calA=\emptyset$}
I have chosen to define the case of $\calA\neq\emptyset$ differently from when $\calA=\emptyset$. 

To give a single definition we could include an additional separate clause for the case where $B\cap\posgambles\neq\emptyset$ \citep{decadt2021decide}. So, we would offer: \begin{quote}
For any $\calA$ (including $\emptyset$), $B\in\extension(\calA)$ iff either: \begin{itemize}
\item There is some $f\in B$ where also $f\in\posgambles$. Or
\item There are some $A_1,\ldots,A_n\in\calA$ such that for each sequence $\seq{g_1,\ldots,g_n}\in  A_1\times\ldots\times A_n$, whenever $0\notin \desext{\Set{g_1,\ldots,g_n}}$, there is some $f_{\seq{g_1,\ldots,g_n}}\in B$ where also $f_{\seq{g_1,\ldots,g_n}}\in\desext{\Set{g_1,\ldots,g_n}}$. 
\end{itemize}
\end{quote}

An alternative would be to first add something that we know to be in all coherent $\K$ so that we essentially ensure we're working with non-empty $\calA$, for example we could add $\posgambles$ itself or each singleton $\set{g}$ for $g\in\posgambles$. This would then allow us to use the \cref{thm:coherentclosure} characterisation. But I want to keep the non-empty case as simple as possible as it's really the interesting one, so I keep them separate.

Both of these have the disadvantage of introducing additional clauses to be checked. I have thus opted to keep the definition simpler and simply deal with $\calA=\emptyset$ manually.

\subsubsection{Removing the use of $\posgambles$ in the $\posi$:}

Decant doesn't work with $\posi(\set{g_1,\ldots,g_n}\cup\posgambles)$ but instead simply with $\posi$ and $\geq$. His definition is (with different formulation):
\begin{quote}
For any $\calA$ (including $\emptyset$), $B\in\extension(\calA)$ iff either: \begin{itemize}
\item There is some there are some $f\in B$ where also $f\in\posgambles$. Or
\item There are some $A_1,\ldots,A_n\in\calA$ such that for each sequence $\seq{g_1,\ldots,g_n}\in  A_1\times\ldots\times A_n$, whenever $0\notin \desext{\Set{g_1,\ldots,g_n}}$, there is some $f_{\seq{g_1,\ldots,g_n}}\in B$ and $h_{\seq{g_1,\ldots,g_n}}\in\posi(\set{g_1,\ldots,g_n})$ where $f_{\seq{g_1,\ldots,g_n}}\geq h_{\seq{g_1,\ldots,g_n}}$.
\end{itemize}
\end{quote}

\begin{proposition}
This is equivalent to \cref{def:natext}
\end{proposition}
\begin{proof}
\begin{sublemma}
$f\in \desext{\Set{g_1,\ldots,g_n}}$ iff  $f\in\posgambles$ or there is some $h\in\posi(\Set{g_1,\ldots,g_n})$ with $f\geq h$.
\end{sublemma}
\begin{proof}
The left-to-right direction is \cref{thm:coherentclosure:lem}. 

The right-to-left direction is simple: When $f\in\posgambles$ clearly $f\in\desext{\set{g_1,\ldots,g_n}}$. When $f\geq h$ with $h\in\posi(\Set{g_1,\ldots,g_n})$, then $f-h\geq 0$. If $f=h\in\posi(\Set{g_1,\ldots,g_n})$, then also $f\in\desext{\Set{g_1,\ldots,g_n}}$. When $f\neq h$, then $f-h\in\posgambles$, so $f=h+(f-h)\in\desext{\set{g_1,\ldots,g_n}}$.
\end{proof}

Thus, \cref{def:natext} is equivalent to:
\begin{quote}
When $\calA\neq\emptyset$: \begin{itemize}
\item $B\in\extension(\calA)$ iff there are some $A_1,\ldots,A_n\in\calA$ such that for each sequence $\seq{g_1,\ldots,g_n}\in  A_1\times\ldots\times A_n$, whenever $0\notin \desext{\Set{g_1,\ldots,g_n}}$, there is some $f_{\seq{g_1,\ldots,g_n}}\in B$ where \begin{itemize}
\item either $f_{\seq{g_1,\ldots,g_n}}\in\posgambles$ 
\item or there is some $h_{\seq{g_1,\ldots,g_n}}\in\posi(\Set{g_1,\ldots,g_n})$ with $f_{\seq{g_1,\ldots,g_n}}\geq h_{\seq{g_1,\ldots,g_n}}$.
\end{itemize}
\end{itemize}
When $\calA=\emptyset$,\begin{itemize}\item  $B\in\extension(\emptyset)$ iff there is some $f\in B$ where also $f\in\posgambles$. 
\end{itemize}
\end{quote}
The first condition of the $\calA\neq\emptyset$ case and the $\calA=\emptyset$ case can be combined into the single condition that $B\cap\posgambles\neq\emptyset$. We are then left with the alternative condition that \begin{quote}
There are some $A_1,\ldots,A_n\in\calA$ such that for each sequence $\seq{g_1,\ldots,g_n}\in  A_1\times\ldots\times A_n$, whenever $0\notin \desext{\Set{g_1,\ldots,g_n}}$, there is some $f_{\seq{g_1,\ldots,g_n}}\in B$ and $h_{\seq{g_1,\ldots,g_n}}\in\posi(\set{g_1,\ldots,g_n})$ where $f_{\seq{g_1,\ldots,g_n}}\geq h_{\seq{g_1,\ldots,g_n}}$.
\end{quote}
So we see that the two formulations are equivalent.
\end{proof}

Note that it is important in doing this that the additional clause about $f\in\posgambles$ is added. 
That is, we cannot define $\extension(\calA)$ by simply taking our definition (\cref{def:natext}) and replacing the existence of $f\in B\cap \desext{\Set{g_1,\ldots,g_n}}$ with the existence of $f\in B$ and $h\in\posi(\set{g_1,\ldots,g_n})$ with $f\geq h$. [However, interestingly if we move to strict dominance in our criteria that would be possible.]

\subsection{Relationship to \citet{debock2018desirability}}\label{sec:relntodebock}

\citet{debock2018desirability} also give a formulation of the natural extension. 
 Since they both characterise natural extensions, they'll be equivalent, but in this section we consider the difference in more detail.

They: 
\begin{enumerate}
\item Add the singletons from $\posgambles$ to $\calA$. This gives us $\calA\cup\set{\set{g}\given g\in\posgambles}$.\label{itm:debocknatext:1}
\item We add sets using $\posi$ from these sets. More carefully: For any $C_1,\ldots,C_n$ in $\calA\cup\set{\set{g}\given g\in\posgambles}$ (from \ref{itm:debocknatext:1}), if for each sequence $\seq{g_i}\in\bigtimes C_i$, $f_{\seq{g_i}}\in \posi(\set{g_1,\ldots,g_n})$, then we add $\set{f_{\seq{g_i}}\given \seq{g_i}\in\bigtimes C_i}$. Our resultant set is called $\mathrm{Posi}(\calA\cup\set{\set{g}\given g\in\posgambles})$.\label{itm:debocknatext:2}
\item
Add any sets obtained by removing some gambles which are $\leq 0$. \label{itm:debocknatext:3}
\item Add any supersets thereof. \label{itm:debocknatext:4}
\end{enumerate}
This gives us the final set, called:  $\mathrm{Rs}(\mathrm{Posi}(\calA\cup\set{\set{g}\given g\in\posgambles}))$. They show that this is the natural extension in the finite setting. So in this setting, we know that it will be equivalent to ours, however it's open whether it gives the natural extension in the infinite setting.

It's got the same sort of moving parts as our definition: there's addition of some $\posgambles$ (in \ref{itm:debocknatext:1}), there's taking posi's (in \ref{itm:debocknatext:2}), there's removing some obviously-bad stuff (in \ref{itm:debocknatext:3}) and taking supersets (in \ref{itm:debocknatext:4}). These are the same sorts of moving parts that I have.

Let's try to rework their formulation to present it in a way that is more similar to our presentation so we can evaluate how/if they are different. Firstly, we can summarise the construction to: 
\begin{itemize}
\item  $B\in\extension_{\mathrm{DBdC}}(\calA)$ iff there are some $C_1,\ldots, C_n\in\calA\cup\set{\set{g}\given g\in \posgambles}$ with $\seq{g_1,\ldots,g_n}\in  C_1\times\ldots\times C_n$, there is some $f_{\seq{g_1,\ldots,g_n}}\in B\cup\G_{\leq 0}$ where also $f_{\seq{g_1,\ldots,g_n}}\in\posi(\Set{g_1,\ldots,g_n})$. 
\end{itemize}

We can then split the choice of $C_1,\ldots,C_n$ from $\calA\cup\set{\set{g}\given g\in \posgambles}$ into a choice separately of $A_1,\ldots,A_n\in\calA$, and $p_1,\ldots,p_m\in\posgambles$. Sometimes no members of $\calA$ are chosen, which should be presented as a separate clause. So we see that:
\begin{itemize}
\item $B\in\extension_{\mathrm{DBdC}}(\calA)$ iff either:\begin{itemize}
\item There are some $A_1,\ldots,A_n\in\calA$ and $p_1,\ldots,p_m\in\posgambles$ such that for each sequence $\seq{g_1,\ldots,g_n}\in  A_1\times\ldots\times A_n$, there is some $f_{\seq{g_1,\ldots,g_n}}\in B\cup\G_{\leq 0}$ where also $f_{\seq{g_1,\ldots,g_n}}\in\posi(\Set{g_1,\ldots,g_n,p_1,\ldots,p_m})$.
\item There are some $p_1,\ldots,p_m\in\posgambles$ such that there is some $f\in B\cup\G_{\leq 0}$ where also $f\in\posi(\Set{p_1,\ldots,p_m})$.
\end{itemize}
\end{itemize}

Note that this is different to our:

\begin{itemize}
\item $B\in\extension(\calA)$ iff there are some $A_1,\ldots,A_n\in\calA$ and for each sequence $\seq{g_1,\ldots,g_n}\in  A_1\times\ldots\times A_n$, there is some $f_{\seq{g_1,\ldots,g_n}}\in B\cup\set{0}$ where also $f_{\seq{g_1,\ldots,g_n}}\in\desext{\Set{g_1,\ldots,g_n}}$. 
\end{itemize}

Some of the differences between these are simply choices of presentation.

However, substantial differences are:
\begin{itemize}
\item DBdC allow the choice of the sequence $C_1,\ldots,C_n$ to contain no members of $\calA$ whatsoever. This allows their definition to apply properly to the case of $\calA=\emptyset$, whereas we need a separate definition for that case. 
\item   Their $\mathsf{Posi}$ is defined working with finitely many $C_1,\ldots,C_n$, which means in any use of it we can only use finitely many members of $\posgambles$. 
Is that enough? It's closely related to my addition of axiom \ref{ax:sets:dominators}.

For example, suppose $A\in\calA$ and we have for each $g$ in $A$ a distinct $p_g\in\posgambles$, we want to ensure that $\set{g+p_g\given g\in A}\in\extension(\calA)$. 

This doesn't immediately follow from the construction of $\mathsf{Posi}$ because it involves infinitely many sets: $A$ plus the infinitely many singletons $\set{p_g}$. 

In fact, like the dominance axiom, when $\Omega$ is finite it'll be fine because $g+p_g\in \posi(\set{g,\indicator{\omega_1},\ldots,\indicator{\omega_n}})$ so we can just work with the finitely many sets $A,\set{\indicator{\omega_1}},\ldots,\set{\indicator{\omega_n}}$. 
But perhaps when $\Omega$ is infinite it'll be different.

\end{itemize}

\begin{suggestion}
\subsection{Does DBdC get infinite addition axiom?}

Should DBdC extend the construction to define $\mathsf{Posi}$ that it directly works with infinitely many sets? 
\end{suggestion}

\subsection{Extending}

\subsubsection{Regularity}
We can replace $\posgambles$ by $\strposgambles$ and everything still works.  \Cref{thm:coherentclosure} then becomes simpler because any $f_{\seq{g_i}}\in \posi(\Set{g_1,\ldots,g_n}\cup\G_{> 0})$, has some   $h_{\seq{g_i}}\in\posi(\Set{g_1,\ldots,g_n})$ with $f_{\seq{g_i}}\geq h_{\seq{g_i}}$ so we don't need a separate clause. 

In fact this shows us that in the strict-dominance setting, axiom \ref{ax:sets:pos} is not required as a separate axiom, though one should then include a non-trivial axiom that $\K$ is non-empty.

\subsubsection{Extending to infinite addition}

If we extend to:
\begin{axioms}[K]
\axiom{\mathrm{AddInf}}{ If $A_i\in\K$ for each $i\in I$ (possibly infinite) and for each sequence $\seq{g_i}_{i\in I}$ with each $g_i\in A_i$, $f_{\seq{g_i}}$ is some member of $\posi(\Set{g_i\given i \in I})$, then 
$$
\Set{f_{\seq{g_i}}\given \seq{g_i}\in\bigtimes_{i\in I} A_i}\in\K
$$\label{ax:sets:infadd}
}
\end{axioms}
we obtain a strictly stronger system \citep[\S4.4.1]{campbellmoore2021isipta}. This resultant axiom system is an instance of the axioms in \citet{debock2023desirablethings}.

The results immediately carry through now instead working with: 
\begin{itemize}
\item $B\in\extension(\calA)$ iff there are some $A_i\in\K$ (a possible infinite collection) where for each sequence $\seq{g_i}\in \bigtimes_i A_i$, there is some $f_{\seq{g_i}}\in B\cup\set{0}$ with $f_{\seq{g_i}}\in \desext{\Set{g_i\given i\in I}}$. 
\end{itemize}

\section{Axioms when gamble sets must be finite}\label{sec:whenfinite}
This section shows that in the special case when gamble sets are finite, we can use the axioms from \cite{debock2018desirability}

\begin{definition}
$\K\subseteq\wp^{\mathrm{finite}}(\G)$ is \emph{coherent} if it satisfies \ref{ax:sets:nontrivial} \ref{ax:sets:0} \ref{ax:sets:pos} \ref{ax:sets:supersets} and:
 \begin{axioms}[K]%
 \axiom{\mathrm{AddPair}}{If $A_1,A_2\in\K$ and for each pair $g_1\in A_1$ and $g_2\in A_2$, $f_{\seq{g_1,g_2}}$ is some member of $\posi(\Set{g_1,g_2})$, then $$\Set{f_{\seq{g_1,g_2}}\given g_1\in A_1,g_2\in A_2}\in\K$$\label{ax:sets:addpair}}\qedhere
 \end{axioms}
\end{definition}

These are equivalent to our earlier axioms when we are restricted to finite sets. The proof of this is contained in the next two sections.

\subsection{Finite addition axiom - axiom \ref{ax:sets:add} }

\begin{proposition}\footnote{Thanks to Arthur van Camp for a key insight, and Jasper De Bock and Arne Decadt for discussion. }Assume we are restricted to finite sets.  Axiom \ref{ax:sets:add} follows from axiom \ref{ax:sets:addpair} and axiom \ref{ax:sets:supersets}.
\end{proposition}
\begin{proof}
We in fact prove:%
\begin{itemize}
\item If $A_1,\ldots,A_n\in\K$ and for each sequence $\seq{g_1,\ldots,g_n}\in  A_1\times\ldots\times A_n$, there is some $f_{\seq{g_1,\ldots,g_n}}\in B\cap\posi(\Set{g_1,\ldots,g_n}$; \\Then $B\in\K$.
\end{itemize}
which entails axiom \ref{ax:sets:add}. (It essentially just adds any supersets. This allows us to avoid some fiddliness in the proof.)

We work by induction on $n$. 

Base case: For $n=1$ use axiom \ref{ax:sets:addpair} and axiom \ref{ax:sets:supersets}, letting $A_1=A_2$. 

Inductive step: Assume it holds for $n$ and consider $n+1$.
 
So, suppose $A_1,\ldots,A_n,A_{n+1}\in\K$ and for each sequence $\seq{g_i}\in \bigtimes_i A_i$ and $a\in A_{n+1}$, there is some $f_{\seq{g_1,\ldots,g_n,a}}\in B\cap \posi(\set{g_1,\ldots,g_n,a})$. We need to show that $B\in\K$.

We will abuse notation and write $f_{\seq{g_i,a}}$ instead of $f_{\seq{g_1,\ldots,g_n,a}}$.

$A_{n+1}$ is finite, so enumerate it, $A_{n+1}=\set{a_1,\ldots,a_K}$.
Define:
\begin{equation}
B_{k}:=\set*{f_{\seq{g_i,a}}\given  \seq{g_i}\in \bigtimes_{i\leq n} A_i,\; a\in\set{a_1,\ldots,a_k}}\cup\set{a_{k+1},\ldots,a_K}\label{eq:finaddn:B}
\end{equation} 
Where we have already replaced $k$-many of $A_{n+1}$ with the relevant members of $B$. 
We show by a induction on $k$ that $B_k\in\K$. Since $B_K\subseteq B$, this will suffice.\footnote{Arthur van Camp's insight that this argument can be done in such a step-by-step way to just use axiom \ref{ax:sets:addpair} when the sets are finite, i.e., it can be shown by a sub-induction.}

Base case: $B_0=A_{n+1}\in\K$.

Inductive step: we can assume that 
$B_{k}\in\K$ and need to show $B_{k+1}\in \K$.

We will construct a set $C$ which we know to be in $\K$, which we can then combine with $B_k$ using axiom \ref{ax:sets:addpair} (and axiom \ref{ax:sets:supersets}) to get $B_{k+1}\in\K$.

By assumption,
$f_{\seq{g_i,a_{k+1}}}\in\posi(\set{g_1,\ldots,g_n,a_{k+1}})$.
 So there is some $h_{\seq{g_i,a_{k+1}}}$ where:
 \begin{align}
 &h_{\seq{g_i,a_{k+1}}}\in \posi(\Set{g_1,\ldots,g_n}),\text{ and }\label{eq:finaddn:h}\\
 &f_{\seq{g_i,a_{k+1}}}\in \posi(\Set{a_{k+1},h_{\seq{g_i,a_{k+1}}}}). \label{eq:finaddn:f}
 \end{align}
To see this: $f_{\seq{g_i,a_{k+1}}}$ has the form $\sum_{i\leq n}\lambda_ig_i+\mu a_{k+1}$. If $\lambda_i> 0$ for some $i$, then we put $h_{\seq{g_i,a_{k+1}}}:=\sum_{i\leq n}\lambda_ig_i$. If not, then we simply let it be, for example, $g_1$. This is as required.

We let \begin{equation}
C:=\set{h_{\seq{g_i,a_{k+1}}}\given \seq{g_i}\in \bigtimes_{i\leq n} A_i}
\end{equation}
$C\in\K$ by the induction hypothesis (on $n$), using \cref{eq:finaddn:h}.

We will show that for every $\seq{c,b}\in B_k\times C$, there is some  $d\in  \posi(\set{c,b})\cap B_{k+1}$; which will then allow us to use axiom \ref{ax:sets:addpair} (and \ref{ax:sets:supersets}) to get that $B_{k+1}\in\K$.
 
If $c\in B_k\cap B_{k+1}$, we just put $d=c$. The only remaining member of $B_k$ to consider is $c=a_{k+1}$. For any $b\in C$, $b=h_{\seq{g_i,a_{k+1}}}$ for some $\seq{g_i}\in \bigtimes_i A_i$. By definition of this (\cref{eq:finaddn:f}),  $f_{\seq{g_i,a_{k+1}}}\in\posi(\set{a_{k+1},b})$. And by looking at the definition of $B_{k+1}$ (\cref{eq:finaddn:B}), also $f_{\seq{g_i,a_{k+1}}}\in B_{k+1}$. So $f_{\seq{g_i,a_{k+1}}}$ is the required $d\in\posi(\set{c,b})\cap B_{k+1}$. 

So, by axiom \ref{ax:sets:addpair} and axiom \ref{ax:sets:supersets},  $B_{k+1}\in\K$, completing the inductive step.

This has shown by induction that $B_K\in\K$. $B_K\subseteq B$, so $B\in\K$ (using \ref{ax:sets:supersets}). 

This completes the inductive step for our initial induction (on $n$). Thus we have shown that our initial statement holds for all $n$. And since it entails axiom \ref{ax:sets:add}, we have shown that this follows.  
\end{proof}

I think that one can do it without axiom \ref{ax:sets:supersets}; though it'll be fiddlier. 
\begin{suggestion}
See  \cref{sec:addpair} for an attempted proof. 
\end{suggestion}

\subsection{Dominators axiom - axiom \ref{ax:sets:dominators}}

\begin{proposition}\label{thm:dom from add}Assume we are restricted to finite sets.  Axiom \ref{ax:sets:dominators} follows from axiom \ref{ax:sets:add} and \ref{ax:sets:pos}.
\end{proposition}
The proof is closely related to \citet[Lemma 34]{debock2019interpreting}. There they show that when one replaces a single member of a desirable gamble set by a dominator it remains desirable. We could directly use this and iterate it (finitely many times) to show our result (to write this down carefully, we'd have to do a proof by induction).  But it is slightly cleaner to directly rely on our already proved axiom \ref{ax:sets:add} and do the replacements simultaneously. 
\begin{proof}

Suppose $B=\set{f_g\given g\in A}$, where $f_g\geq g$. 

Now, $h_g:=f_g-g\in\G_{\geq 0}$. 

If $h_g\neq \zerogamble$, the singleton $\set{h_g}\in\K$. 

So we have $A\in\K$ and each singleton $\set{h_g}\in\K$ whenever $f_g\weakdom g$. 

Since $A$ is finite, this is finitely many sets.

For any $g^*\in A$, $f_{g^*}\in\posi(\set{g^*}\cup\set{h_{g}\given g\in A})$. 

So by axiom \ref{ax:sets:add}, $B\in \K$. 
\end{proof}
We can similarly see:
\begin{proposition}\label{thm:dom from infadd}
Axiom \ref{ax:sets:dominators} follows from axiom \ref{ax:sets:infadd} and axiom \ref{ax:sets:pos} even when sets can be infinite.
\end{proposition}
\begin{proof}
The argument from \cref{thm:dom from add}.
\end{proof}
\section{Representation with desirable gambles}

There is a representation result in the finite setting which says that every coherent $\K$ can be captured by a set of coherent desirable gambles: \citet[Theorem 9]{debock2019interpreting}. We have to generalise this result in the infinite setting if we want to avoid axiom \ref{ax:sets:infadd} and merely impose axiom \ref{ax:sets:add}.

\subsection{Sets of coherent desirable gamble sets}\label{sec:repinf}

If we have a set of desirable gambles, $\des$, we can extract a set of desirable gamble sets, $\K_\des$:  $\des$ evaluates a gamble set $B$ as desirable when it thinks that some member of the set is desirable
\begin{align}
B\in\K_\des&\text{ iff there is some $g\in B$ with $g\in\des$}\\
&\text{ iff }B\cap \des\neq\emptyset
\end{align}
This results in very special kinds of sets of desirable gamble sets. We can get more if we instead look at generating a set of desirable gamble sets by using a  \emph{set} ($\setdes$) of coherent sets of desirable gambles ($\des$s).

If we have a set of  coherent $\des$s, $\setdes$,
we can consider evaluations of a set of gambles, $B$, as desirable if every $\des\in\setdes$ evaluates it as desirable. 
 That is:
\begin{align}
B\in \K_\setdes&\text{ iff $B\in\K_\des$ for each $\des\in\setdes$}
\\&\text{ iff for each $\des\in\setdes$ there is some $g\in B$ with $g\in\des$}
\\&\text{ iff for each $\des\in\setdes$ there is some $g\in B\cap\des$}
\end{align}
More concisely, $\K_\setdes:=\bigcap_{\des\in\setdes} \K_\des$.

When $\setdes$ is a nonempty set of coherent $\des$, 
$\K_\setdes$ is coherent. But it also satisfies axiom \ref{ax:sets:infadd}. 
\begin{proposition}\label{thm:setdes sat infadd}
Let $\setdes$ be a nonempty set of coherent $\des$s. 
 $\K_\setdes$ satisfies K$_\text{infadd}$.\footnote{This result is also \citet[Theorem 9, Example 9]{debock2023desirablethings}}
\end{proposition}
\begin{proof}
Let $A_i\in\K_\setdes$ with $f_{\seq{g_i}}\in\posi(\set{g_i\given i\in I})$ with $B=\set{f_{\seq{g_i}}\given \seq{g_i}\in\bigtimes A_i}$. 

So for each $\des\in\setdes$, $\des\cap A_i\neq \emptyset$. 

Fix any $\des\in\setdes$. Take $g_i$ to be some member of $A_i\cap \des$. 
Now, $f_{\seq{g_i}}\in\posi(\set{g_i\given i\in I})$, so since each $g_i\in\des$, also $f_{\seq{g_i}}\in\des$ by coherence of $\des$. Thus $B\cap \des\neq\emptyset$.

So, for any $\des\in\setdes$, $B\cap \des\neq\emptyset$. Thus, $B\in\K_{\setdes}$. 
\end{proof}
Note that, as we show in \citet[\S4.4.1]{campbellmoore2021isipta}, some coherent $\K$ fail axiom \ref{ax:sets:infadd}
So these are still special kinds of coherent sets of desirable gamble sets.

In fact, $\K_\setdes$ are exactly the coherent $\K$ that satisfy axiom \ref{ax:sets:infadd}.
 See \citet[Theorem 9, Example 9]{debock2023desirablethings} for this result and generalisations. %
 Instead of giving a direct proof here, we move to consider the more general question:
What about when we allow for mere finite addition? Can we generate \emph{all} coherent $\K$ in an analogous way? Answer: yes!

\subsection{Representation of all coherent $\K$ using $\des$s}

A gamble set $B$ was evaluated by $\setdes$ by seeing if every $\des$ in $\setdes$ evaluated $B$ as desirable. Now, we consider different possible evaluations, as given by a collection of $\setdes$, and simply ask that one of them evaluates $B$ as desirable. 

That is, we put: 
\begin{align}
B\in \K_\setsetdes&\text{ iff there is some $\setdes\in\setsetdes$ with $B\in\K_\setdes$}
\\&\text{ iff there is some $\setdes\in\setsetdes$ s.t.~for all $\des\in\setdes$, $B\in\K_\des$ }
\\&\text{ iff there is some $\setdes\in\setsetdes$ s.t.~for all $\des\in\setdes$ there is some $g\in B\cap\des$}\label{eq:Ksetsetdes}
\end{align}
Or, more concisely, 
\begin{equation}
\K_\setsetdes:=\bigcup_{\setdes\in\setsetdes}\bigcap_{\des\in\setdes} \K_\des
\end{equation}

Whereas we went from $\des$ to $\setdes$ using universal quantification, we go from $\setdes$ to $\setsetdes$ using existential quantification.
$\setsetdes$ contains the $\setdes$ that one is happy to evaluate gamble sets with respect to. This is different from a $\setdes$ which contains the $\des$ that one thinks are still open. 
One is not sure about which $\des\in\setdes$ is ``right'' so will only make judgements when agreed on by all $\des\in\setdes$. But each $\setdes\in\setsetdes$ is in some sense good and can be trusted to make decisions.

In general, this will not generate something coherent. In particular, there is no guarantee that axiom \ref{ax:sets:add} will hold, as we can select $A_i$s evaluated as desirable by different $\setdes_i$, and there's no guarantee that the relevant $\set{f_{\seq{g_i}}\given \seq{g_i}\in\bigtimes A_i}$ is evaluated as desirable by any $\setdes$ in $\setsetdes$. To ensure that axiom \ref{ax:sets:add} holds, we will add a requirement on $\setsetdes$: we will require it to be downwards closed:\footnote{I had originally required that it is actually closed under finite intersection. Thanks to Jasper De Bock for noticing that my proof didn't then work. } when $\setdes_1,\setdes_2\in \setsetdes$, there is some $\setdes\in\setsetdes$ with $\setdes\subseteq\setdes_1\cap\setdes_2$. This will ensure that by taking some $A_1$ evaluated as desirable by $\setdes_1$ and $A_2$ by $\setdes_2$, then the $\setdes\subseteq\setdes_1\cap\setdes_2$ will evaluate the required $\set{f_{\seq{g_1,g_2}}\given \seq{g_1,g_2}\in A_1\times A_2}$ as desirable; and it can be extended to finitely many members to get axiom \ref{ax:sets:add}. 

\todo[inline]{Change this to filters not downward closed sets}
This then allows us to represent all coherent $\K$, where we impose axiom \ref{ax:sets:add} but not necessarily axiom \ref{ax:sets:infadd}. 

\begin{theorem}\label{thm:setsetdes rep}
$\K$ is coherent iff there is some $\setsetdes$ a non-empty set of $\setdes$s, which are themselves non-empty sets of coherent $\des$s, and where $\setsetdes$ is downwards closed, with $\K=\K_\setsetdes$, i.e.,:\begin{equation}
B\in \K\text{ iff there is some $\setdes\in\setsetdes$ s.t.~for all $\des\in\setdes$ there is some $g_\des\in B\cap\des$}\label{eq:Ksetsetdes}
\end{equation}
\end{theorem}
\begin{proof}
Suppose $\K$ is coherent. 

Let $\setsetdes_\K:=\set{\setdes_{\set{A_1,\ldots,A_n}}\given A_1,\ldots,A_n\text{ are finitely many members of $\K$}}$ where
\begin{equation}\label{eq:setdesA1An}
\setdes_{\set{A_1,\ldots,A_n}}:=\set*{D\text{ coherent}\given  \begin{split}
&\text{There is $\seq{g_1,\ldots,g_n}\in A_1\times\ldots\times A_n$ with }\\
&0\notin\desext{\set{g_1,\ldots,g_n}}\text{ and } \\
&D\supseteq \desext{\set{g_1,\ldots,g_n}}
\end{split}}
\end{equation}

\begin{sublemma}
$\setsetdes_\K$ is a non-empty set of $\setdes$s, which are themselves non-empty sets of coherent $\des$s, and $\setsetdes_\K$ is downwards closed
\end{sublemma}
\begin{proof}
$\setsetdes_\K$ is non-empty because $\K$ is non-empty and for any $A\in\K$, $\setdes_{\set{A}}\in\setsetdes_\K$. 

Each $\setdes_{\set{A_1,\ldots,A_n}}$ is a non-empty set of coherent $\des$s: It is a set of coherent $\des$s by definition; we just need to show that it is non-empty. 
If there is some $\seq{g_i}\in\bigtimes A_i$ with $0\notin\desext{\set{g_1,\dots, g_n}}$, then $\desext{\set{g_1,\ldots, g_n}}$ is coherent by \cref{thm:des nat extn}. So then $\setdes_{\set{A_1,\ldots,A_n}}$ is non-empty. Otherwise, every $\seq{g_i}\in\bigtimes A_i$ has $0\in\desext{\set{g_1,\ldots,g_n}}$; but in this case, $\emptyset\in\K$ by \cref{thm:coherentclosure}, contradicting the coherence of $\K$.

It remains to show that $\setsetdes_\K$ is downwards closed. We will show that \begin{equation}
\setdes_{A_1^1,\ldots,A_{n_1}^1}\cap\setdes_{A_1^2,\ldots,A_{n_2}^2}\supseteq\setdes_{A_1^1,\ldots,A_{n_1}^1,A_1^2,\ldots,A_{n_2}^2}
\end{equation}

Suppose $\des\in\setdes_{A_1^1,\ldots,A_{n_1}^1,A_1^2,\ldots,A_{n_2}^2}$. So there is $\seq{g_1^1,\ldots, g_{n_1}^1, g_1^2,\ldots, g_{n_2}^2}\in A_1^1\times\ldots\times A_{n_1}^1\times A_1^2\times\ldots\times A_{n_2}^2$ with $0\notin\desext{\set{g_1^1,\ldots, g_{n_1}^1, g_1^2,\ldots, g_{n_2}^2}}$ and $ D\supseteq \desext{\set{g_1^1,\ldots, g_{n_1}^1, g_1^2,\ldots, g_{n_2}^2}}$. 

Since $\desext{\set{g_1^1,\ldots, g_{n_1}^1, g_1^2,\ldots, g_{n_2}^2}}\supseteq\desext{\set{g_1^1,\ldots, g_{n_1}^1}}$, 
also $\des\supseteq \desext{\set{g_1^1,\ldots, g_{n_1}^1}}$ and $0\notin \desext{\set{g_1^1,\ldots, g_{n_1}^1}}$. So $\des\in\setdes_{A_1^1,\ldots,A_{n_1}^1}$. Similarly $\des\in\setdes_{A_1^2,\ldots,A_{n_2}^2}$. Thus $\des\in \setdes_{A_1^1,\ldots,A_{n_1}^1}\cap\setdes_{A_1^2,\ldots,A_{n_2}^2}$, as required. 
\end{proof}

\begin{sublemma}
\begin{align}
B\in \K&\text{ iff there is some $\setdes\in\setsetdes$ s.t.~for all $\des\in\setdes$ there is some $g\in B\cap\des$}\label{eq:Ksetsetdes}
\end{align}
\end{sublemma}
\begin{proof}
\begin{align}
&B\in\K_{\setsetdes_\K}\\&\text{ iff there is some $\setdes\in\setsetdes_\K$ s.t.~for all $\des\in\setdes$, $B\cap\des\neq\emptyset$}\\
&\text{ iff there is some $A_1,\ldots,A_n\in\K$ s.t.~$\setdes\supseteq\setdes_{\set{A_1,\ldots,A_n}}$, and for all $\des\in\setdes$, $B\cap\des\neq\emptyset$}\\
&\text{ iff there is some $A_1,\ldots,A_n\in\K$ s.t.~for all $\des\in\setdes_{\set{A_1,\ldots,A_n}}$, $B\cap\des\neq\emptyset$}\\
&\text{ iff }\begin{aligned}&
\text{there is some $A_1,\ldots,A_n\in\K$ s.t.~for all $\seq{g_i}\in\bigtimes A_i$, }\\
&\quad\text{ $B\cap\posi(\set{g_1,\ldots,g_n}\cup \posgambles)\neq\emptyset$ or $0\in\posi(\set{g_1,\ldots,g_n}\cup \posgambles)$}
\end{aligned}\label{eq:rep:1}
\end{align}
By \cref{thm:coherentclosure}, we know that any $B$ satisfying \cref{eq:rep:1} is in $\K$. It is also easy to see that any $B\in\K$ satisfies \cref{eq:rep:1} (by $n=1$, $A_1=B$). So we have that $\K_{\setsetdes_\K}=\K$.
\end{proof}
\begin{sublemma}
When $\setsetdes$ is a non-empty set of $\setdes$s, which are themselves non-empty sets of coherent $\des$s, and $\setsetdes$ is downwards closed, then $\K_\setsetdes$ is coherent, where $B\in \K_\setsetdes$ iff there is some $\setdes\in\setsetdes$ s.t.~for all $\des\in\setdes$, $B\cap\des\neq\emptyset$.
\end{sublemma} 
\begin{proof}
We check each of the axioms. 
It is only axiom \ref{ax:sets:add} which requires the structure of $\setsetdes$, the rest simply hold because they hold of each $\K_\setdes$. 

Axiom \ref{ax:sets:nontrivial}: 
Every $\setdes\in\setsetdes$ is non-empty by assumption, so has some $\des\in\setdes$. And $\emptyset\cap\des=\emptyset$, so for every $\setdes\in\setsetdes$, there's some $\des\in\setdes$ with $\emptyset\cap \des=\emptyset$; so $\emptyset\notin\K_\setsetdes$.

Axiom \ref{ax:sets:0}: Suppose $A\in\K_\setsetdes$. Then there is some $\setdes$ with $A\in\K_\setdes$. And thus $A\in \K_\des$ for each $\des\in\setdes$.
That is, there is some $g\in A\cap \des$. Since $0\notin\des$, also $(A\setminus\set{0})\cap \des\neq \emptyset$. So $(A\setminus\set{0})\in\K_\setdes$. Thus $(A\setminus\set{0})\in\K_\setsetdes$.

Axiom \ref{ax:sets:pos}: Suppose $g\in\posgambles$. Then $g\in\des$ for each coherent $\des$. Consider any $\setdes\in\setsetdes$. Then $g\in\des$ for each $\des\in\setdes$. Thus $g\in\K_\setdes$; thus $g\in\K_\setsetdes$.

Axiom \ref{ax:sets:supersets}: Suppose $A\in\K_\setsetdes$ and $B\supseteq A$. Then there is some $\setdes$ with $A\in\K_\setdes$. And thus $A\cap\des\neq\emptyset$ for each $\des\in\setdes$. So also $B\cap \des\neq\emptyset$ for each $\des\in\setdes$. So $B\in\K_\setdes$. So $B\in\K_\setsetdes$.

Axiom \ref{ax:sets:dominators}: Suppose $A\in\K_\setsetdes$ and we have $f_g\geq g$ for each $g\in A$. There is some $\setdes$ with $A\in\K_\setdes$. And thus $A\in \K_\des$ for each $\des\in\setdes$. That is, there is some $g\in A\cap \des$. But therefore $f_{g}\in \des$ by properties of coherent $\des$, thus $B\in\K_\des$ for each $\des\in\setdes$. Thus $B\in\K_\setdes$, and so $B\in\K_\setsetdes$.
 
For axiom \ref{ax:sets:add}: 
Let $A_1,\ldots,A_n\in\K_\setsetdes$ and $B=\set{f_{\seq{g_i}}\given \seq{g_i}\in \bigtimes A_i}$ where $f_{\seq{g_i}}\in\posi({\set{g_i}})$. 
Since $A_i\in\K_\setsetdes$, we have some $\setdes_i\in\setsetdes$ with $A_i\in\K_{\setdes_i}$. By downwards closure, there is some $\setdes^*\in\setsetdes$ with $\setdes^*\subseteq\setdes_1\cap\ldots\cap\setdes_n$. We will show that $B\in\K_{\setdes^*}$. For any $\des\in\setdes^*$, also $\des\in\setdes_i$ for each $i$, and thus there is some $g_i\in\des\cap A_i$. Since $\des$ is coherent containing $g_1,\ldots,g_n$ and $f_{\seq{g_i}}\in\posi\set{g_1,\ldots,g_n}$, also $f_{\seq{g_i}}\in\des$. So $\des\cap B\neq\emptyset$. Since this holds for all $\des\in\setdes^*$, $B\in\K_{\setdes^*}$. So $B\in\K_\setsetdes$, as required. 
\end{proof}
This proves our result. 
\end{proof}

It may be more intuitive to think about $\setsetdes$ not simply as a downwards closed set of $\setdes$s, but as a filter of $\setdes$s. That is, we might also close it under supersets. This won't make a difference to the representation, because adding supersets doesn't add any new members to $\K_\setsetdes$.

We can then think of $\setsetdes$ as containing all one's judgements about what the ``correct'' $\des$ is like. This would allow us to give a different formulation. 
\begin{align}
B\in \K_\setsetdes&\text{ iff there is some $\setdes\in\setsetdes$ s.t.~for all $\des\in\setdes$ there is some $g\in B\cap\des$}\\
&\text{ iff there is some $\setdes\in\setsetdes$ s.t.~$\setdes\subseteq\set{\des\given \text{there is some $g\in B\cap \des$}}$}\\
&\text{ iff }\set{\des\given \text{there is some $g\in B\cap \des$}}\in\setsetdes\label{eq:Ksetsetdes2}
\end{align}

This is the version of our representation results which I personally find most attractive. It is related to work in \citet{campbellmoore2021isipta}.

\begin{theorem}
\label{thm:setsetdes rep filter}
$\K$ is coherent iff there is some $\setsetdes$ a proper filter of coherent $\des$s,
that is: $\setsetdes\subseteq\powerset{\set{D\subseteq\G\given D\text{ is coherent}}}$ where:
\begin{itemize}[noitemsep]
\item $\setsetdes\neq\emptyset$, $\emptyset\notin\setsetdes$,
\item If $\setdes_1,\setdes_2\in\setsetdes$ then $\setdes_1\cap\setdes_2\in\setsetdes$.
\item If $\setdes_1\in\setsetdes$ and $\setdes_2\supseteq\setdes_1$ then $\setdes_2\in\setsetdes$.
\end{itemize}
 with $\K=\K_\setsetdes$, i.e.,:
 \begin{align}
B\in \K&\text{ iff there is some $\setdes\in\setsetdes$ s.t.~for all $\des\in\setdes$ there is some $g_\des\in B\cap\des$}\label{eq:Ksetsetdes_filter}
\end{align} 
or equivalently:
\begin{align}
B\in \K&\text{ iff }\set{\des\given \text{there is some $g_\des\in B\cap \des$}}\in\setsetdes\label{eq:Ksetsetdes2}
\end{align} 
\end{theorem}
\begin{proof}
When $\setsetdes$ is a non-empty set of $\setdes$s, which are themselves non-empty sets of coherent $\des$s, and $\setsetdes$ is downwards closed, then by closing it under supersets gets a proper filter of coherent $\des$s. Also observe that adding supersets does not add any new gamble sets satisfying \cref{eq:Ksetsetdes_filter}. 
So the representation of coherent $\K$ by such a $\setsetdes$ from \cref{thm:setsetdes rep} can be closed under supersets to get $\K$ represented by a proper filter of coherent $\des$s.

The right-to-left holds because a filter of coherent $\des$s is itself a is a non-empty set of $\setdes$s, which are themselves non-empty sets of coherent $\des$s, and $\setsetdes$ is downwards closed.

For the equivalence between \cref{eq:Ksetsetdes_filter,eq:Ksetsetdes2}, note that if there is some such $\setdes\in\setsetdes$ then $\set{\des\given \text{there is some $g_\des\in B\cap \des$}}$ is a superset of it, so is also in $\setsetdes$.
\end{proof}

Note that one thing that this result does not show is an equivalence between $\setsetdes$ and $\K$s. Multiple $\setsetdes$ will correspond to the same $\K$.\footnote{See also \citet[section 5]{campbellmoore2021isipta} for the related observation in that more restrictive setting.} See \citet{decooman2023filterdesirable}.

\begin{suggestion}
\begin{proposition}
The minimal filter $\setsetdes$ corresponding to coherent $\K$ is  $\setsetdes_\K$ where $\setdes\in\setsetdes_\K$ iff there are finitely many members $A_1,\ldots, A_n\in\K$ with $\setdes\supseteq\setdes_{\set{A_1,\ldots,A_n}}$ (in the sense of \cref{eq:setdesA1An}).
\end{proposition}
\begin{proof}
Should follow very quickly. 
\end{proof}

\begin{proposition}
[Gert and Arthur]
There is an isomorphism between coherent $\K$s and $\setsetdes$ that satisfy:
\begin{itemize}
\item $\setsetdes$ is a non-empty set of non-empty sets of coherent $\des$ closed under finite intersection and superset. 
\item  $\setdes\in\setsetdes$ iff $\uparrow\setdes\in\setsetdes$, where $\uparrow\setdes=\set{\des'\given\exists\des\in\setdes,\;\des'\subseteq \des}.$
\end{itemize}
\end{proposition}
A version of Arthur and Gert's result:
\begin{proposition}
There is an isomorphism between coherent $\K$s and $\setsetdes$ that satisfy:\begin{itemize}
\item $\setsetdes$ is a non-empty set of non-empty sets of coherent $\des$ closed under finite intersection and superset. 
\item If $\setdes\in\setsetdes$ then there are sets of gambles $A_1,\ldots,A_n$ such that $\setdes\supseteq\setdes_{\set{A_1,\ldots,A_n}}$; i.e., where for every coherent $\des$ such that there are $\seq{g_1,\ldots,g_n}\in A_1\times\ldots\times A_n$ with $0\notin \posi(\set{g_1,\ldots,g_n}\cup\posgambles)$ and $\des\supseteq\posi(\set{g_1,\ldots,g_n}\cup\posgambles)$ is itself $\in\setdes$.
\item If $\setdes\in\setsetdes$ and $\setdes'$ is s.t.~$\uparrow \setdes=\uparrow\setdes'$ then $\setdes'\in\setsetdes$. 
\end{itemize} 
\end{proposition}
\end{suggestion}

There is a further alternative we might do which removes requirements on $\setsetdes$: we can drop the requirement that $\setsetdes$ be downwards closed, and instead bake that into the characterisation of $\K_\setsetdes$. We could say: 
\begin{align}
B\in \K_\setsetdes&\text{ iff }\begin{split}
&\text{there are some $\setdes_1,\ldots,\setdes_n\in\setsetdes$}\\&\text{ s.t.~for all $\des\in\setdes_1\cap\ldots\cap\setdes_n$, there is some $g_\des\in B\cap\des$}
\end{split}
\end{align}
Now we get axiom \ref{ax:sets:add} for free, however axiom \ref{ax:sets:0} might now fail. We still need a ``coherence'' requirement on this. The relevant requirement is that $\setsetdes$ has the finite intersection property: that we never have any $\setdes_1\cap\ldots\cap\setdes_n=\emptyset$. It is then equivalent.

\begin{corollary}
\label{thm:setsetdes rep fip}
$\K$ is coherent iff there is some $\setsetdes$ a non-empty set of $\setdes$s, which are themselves non-empty sets of coherent $\des$s, where $\setsetdes$ has the finite intersection property (i.e., no $\setdes_1\cap\ldots\cap \setdes_n=\emptyset$), and \begin{align}
B\in \K&\text{ iff }\begin{split}
&\text{there are some $\setdes_1,\ldots,\setdes_n\in\setsetdes$}\\&\text{ s.t.~for all $\des\in\setdes_1\cap\ldots\cap\setdes_n$, there is some $g_\des\in B\cap\des$}
\end{split}
\end{align}
\end{corollary}
\begin{proof}
If $\setsetdes$ is such a set which has the finite intersection property, consider $\setsetdes'$ which closes it under finite intersections, i.e., $\setsetdes':=\set{\setdes_1\cap\ldots \cap\setdes_n\given \text{$\setdes_1,\ldots,\setdes_n$ are finitely many members of $\setsetdes$}}$. By assumption that $\setsetdes$ has the finite intersection property, $\setsetdes'$ doesn't contain the emptyset. It immediately satisfies all the other assumptions of \cref{thm:setsetdes rep}. 

Also observe that there are some $\setdes_1,\ldots,\setdes_n\in\setsetdes$ s.t.~for all $\des\in\setdes_1\cap\ldots\cap\setdes_n$, $\emptyset\neq B\cap\des$ iff there is some $\setdes'\in\setsetdes'$ s.t.~for all $\des\in\setdes'$, $\emptyset\neq B\cap\des$.

So we see that for any such $\setsetdes$ with the finite intersection property, we can simply close it under finite intersections and use the representation of \cref{thm:setsetdes rep}.
\end{proof}

\begin{mythoughts}

\subsection{Representation result when $\K$ satisfies Infinite Addition}

\begin{suggestion}
Needs redoing given the change to definitions!
\end{suggestion}
\begin{theorem}\label{thm:setdes rep}
$\K$ is coherent satisfying axiom \ref{ax:sets:infadd} iff there is some non-empty set of coherent $\des$s, $\setdes$, where $\K=\K_{\setdes}$. 
\end{theorem}
This result is also \citet[Theorem 2]{debock2023desirablethings}
\begin{proof}
Observe that the singleton $\set{\setdes}$ satisfies property X and that $\K_{\setdes}=\K_{\set{\setdes}}$, so by \cref{thm:setsetdes rep,thm:setdes sat infadd}, $\K_{\setdes}$ is coherent satisfying axiom \ref{ax:sets:infadd}. 

The result then follows from:
\begin{sublemma}
Given any coherent $\K$ satisfying \ref{ax:sets:infadd}. For \begin{equation}
\setdes=\set*{\desext{\set{g_i\given i\in I}}\given \begin{split}
&\seq{g_i}\in\bigtimes_{i\in I} A_i \text{, where the index set $I$ may be infinite}\\
&0\notin\desext{\set{g_i\given i\in I}} 
\end{split}}
\end{equation}we have $\K=\K_{\setdes}$ 
\end{sublemma}
\begin{proof}
\begin{align}
&B\in\K_\setdes\\&\text{ iff for all $\des\in\setdes$, $B\cap\des\neq\emptyset$}\\
&\text{ iff there is some collection $A_i\in\K$ s.t.~for all $\des\in\setdes_{\set{A_i\given i\in I}}$, $B\cap\des\neq\emptyset$}\\
&\text{ iff }\begin{aligned}&
\text{there is some collection $A_i\in\K$ s.t.~for all $\seq{g_i}\in\bigtimes A_i$, }\\
&\quad\text{ $B\cap\posi(\set{g_i\given i\in I}\cup \posgambles)\neq\emptyset$ or $0\in\posi(\set{g_i\given i\in I}\cup \posgambles)$}
\end{aligned}\label{eq:rep:1inf}
\end{align}
By an infinite analogue of \cref{thm:coherentclosure}, \begin{suggestion}
which I didn't explicitly show
\end{suggestion} we know that any $B$ satisfying \cref{eq:rep:1inf} is in $\K$. It is also easy to see that any $B\in\K$ satisfies \cref{eq:rep:1inf}. So we have that $\K_\setsetdes=\K$.
\end{proof}
\end{proof}

\section{Representing with probability filters}
What does this tell us about representing all choice functions with probability filters? These are a set ($\filter$) of sets ($P$) of probabilities ($p$). 

Instead of now, a set ($\setsetdes$) of sets ($\setdes$) of sets of desirable gambles ($\des$).

$p$ are equivalent to special kinds of $\des$:
the open half-spaces. Can all the flexibility be reduced to such $\des$?

So they seem closely related. But it's basically choice by maximality that we've been using. And I thought that gave us infinite addition.... ??? 

Let's do the equivalent thing:

\begin{align}
B\in \K_p&\text{ iff }B\in\des_p\\
&\text{ iff there is some $g\in B$ with }p\cdot g>0
\\
B\in \K_P&\text{ iff $B\in\des_p$ for all $p\in P$}\\
&\text{ iff for all $p\in P$ there is some $g\in B$ with $p\cdot g>0$}
\\
B\in \K_\filter&\text{ iff there is some $P\in\filter$ with $B\in\K_P$}\\
&\text{ iff there is some $P\in\filter$ with $B\in\des_p$ for all $p\in P$}\\
&\text{ iff there is some $P\in\filter$ where for all $p\in P$ there is some $g\in B$ with $p\cdot g>0$}\\
&\text{ iff there is some $P\in\filter$ where $P\subseteq\set{p\given \text{for some $g\in B$, } p\cdot g>0}$}\\
&\text{ iff $\set{p\given\text{for some $g\in B$, } p\cdot g>0}\in\filter$}
\end{align}

Hum... this is just E-admissibility, and gets us the mixing axiom. Weird! So reducing to the open half-spaces reduces flexibility and gets the mixing axiom. This is because the a $\des$ satisfies mixing itself iff it is a half-space. So by restricting to the $\des_p$, which are half-spaces, we get the mixing axiom. 

That's kind of disappointing. 

What do arbitrary $\des$ correspond to? Well... filters. But less than that...

Alternatively, let's see how to write the $\des$ thing in filter-form. It is a filter of sets of desirable gambles. Does that help?

Suppose $\setsetdes$ is a filter. 
\begin{align}
B\in \K_\setsetdes&\text{ iff there is some $\setdes\in\setsetdes$ with $B\in\K_\setdes$}
\\&\text{ iff there is some $\setdes\in\setsetdes$ s.t.~for all $\des\in\setdes$, there is some $g\in B\cap\des$}
\\&\text{ iff there is some $\setdes\in\setsetdes$ s.t.~$\setdes\subseteq\set{\des\given \text{there is some $g\in B\cap \des$}}$}\\
\\&\text{ iff there is some $\setdes\in\setsetdes$ s.t.~$\setdes\subseteq\set{\des\given \text{for some $g\in B$, }g\in\des}$}\\
\\&\text{ iff }\set{\des\given \text{for some $g\in B$, }g\in\des}\in\setsetdes
\end{align}

\end{mythoughts}

\bibcommand
\appendix
\begin{suggestion}

\section{Finite addition proof without supersets}\label{sec:addpair}
 \begin{proposition}
Assume we are restricted to finite sets.  Axiom \ref{ax:sets:add} follows from axiom \ref{ax:sets:addpair}.
 \end{proposition}

\todo[inline]{  Jasper: the pf doesn't work: what if multiple of the f's are all 0?}

 \begin{proof attempt}

  The argument gets slightly fiddly because members may coincide and repetitions get ``deleted'' by sets. To avoid this issue, we need to include constants throughout to ensure distinctness.
  
  We work by induction on $n$. 
  
  Base case: $n=1$ follows directly from axiom \ref{ax:sets:addpair} by putting $A_1=A_2$. 
  
  Inductive step: Assume it holds for $n$ and consider $n+1$.
  
  So, suppose $A_1,\ldots,A_n,A_{n+1}\in\K$ and for each sequence $\seq{g_i}\in \bigtimes_i A_i$ and $a\in A_{n+1}$, we have $f_{\seq{g_1,\ldots,g_n,a}}\in \posi(\set{g_1,\ldots,g_n,a})$ with \begin{equation}
  B=\set{f_{\seq{g_1,\ldots,g_n,a}}\given \seq{g_i}\in \bigtimes_{i\leq n} A_i,\; a\in A_{n+1}}
  \end{equation}  
  We need to show that $B\in\K$.
  
  We will abuse notation and write $f_{\seq{g_i,a}}$ instead of $f_{\seq{g_1,\ldots,g_n,a}}$.
  
  Consider adding some scalar constants:
  let
  \begin{equation}
  f_{\seq{g_i,a}}^*:=\delta_{\seq{g_i,a}}\times f_{\seq{g_i,a}}
  \end{equation}
  where $\delta_{\seq{g_i,a}}$ are chosen to ensure that these are all distinct, and distinct from any $a'$.
  We will show that 
  \begin{equation}
  B^*=\set{f^*_{\seq{g_i,a}}\given \seq{g_i}\in \bigtimes_i A_i}
  \end{equation}    
  
  $A_{n+1}$ is finite, so enumerate it, $A_{n+1}=\set{a_1,\ldots,a_K}$.
  Define:
  \begin{equation}
  B^*_{k}:=\set{f^*_{\seq{g_i,a_j}}\given  \seq{g_i}\in \bigtimes_{i\leq n} A_i,\; a_j\in\set{a_1,\ldots,a_k}}\cup\set{a_{k+1},\ldots,a_K}\label{eq:finaddn:B2}
  \end{equation} 
  Where we have already replaced $k$-many of $A_{n+1}$ with the relevant members of $B^*$. 
  We show by a induction on $k$ that $B^*_k\in\K$.\footnote{Arthur van Camp's insight that this can be done in such a step-by-step way to just use axiom \ref{ax:sets:addpair}, i.e., it can be shown by a sub-induction.}
  
  Base case: $B_0=A_{n+1}\in\K$.
  
  Inductive step: we can assume that 
  $B_{k}\in\K$ and need to show $B_{k+1}\in \K$.

  We will construct a set $C$ which we know to be in $\K$, which we can then combine with $B_k$ using axiom \ref{ax:sets:addpair} (and axiom \ref{ax:sets:supersets}) to get $B^*_{k+1}\in\K$.

  \begin{sublemma}
  We can find distinct $h^*_{\seq{g_i,a_{k+1}}}$ where:
     \begin{align}
     &h^*_{\seq{g_i,a_{k+1}}}\in \posi(\Set{g_1,\ldots,g_n}),\text{ and }\label{eq:finaddn:h2}\\
     &f^*_{\seq{g_i,a_{k+1}}}\in \posi(\Set{h^*_{\seq{g_i,a_{k+1}}},a_{k+1}}). \label{eq:finaddn:f2}
     \end{align}
  \end{sublemma}
  \begin{proof}
  $f_{\seq{g_i,a_{k+1}}}$ has the form $\sum_{i\leq n}\lambda_ig_i+\mu a_{k+1}$. If $\lambda_i> 0$ for some $i$, then we put $h_{\seq{g_i,a_{k+1}}}:=\sum_{i\leq n}\lambda_ig_i$. If not, then we simply let it be, for example, $g_1$. 
  
  To ensure they are distinct, we simply multiply the $h$ by appropriate constants to find $h^*$.
  \end{proof}

  We let \begin{equation}
  C:=\set{h^*_{\seq{g_i,a_{k+1}}}\given \seq{g_i}\in \bigtimes_i A_i}
  \end{equation}
  $C\in\K$ by the induction hypothesis (on $n$), using \cref{eq:finaddn:h2}.
  
  Now, for each $\seq{c,b}$, let 
  \begin{eqnarray}
  e_{\seq{c,b}}:=
  \begin{cases}
  c&b\in B^*_k\cap B^*_{k+1}\\
  f^*_{\seq{g_i,a_{k+1}}}&c=a_{k+1}\text{ and }b=h^*_{\seq{g_i,a_{k+1}}}
  \end{cases}
  \end{eqnarray}
  Since each $e_{\seq{c,b}}\in\posi(\set{c,b})$, by axiom \ref{ax:sets:addpair}, 
  \begin{equation}
  \set{e_{\seq{c,b}}\given c\in B_k,\, b\in C}\in\K
  \end{equation}
  \begin{sublemma}
  \begin{equation}
  B^*_{k+1}=\set{e_{\seq{c,b}}\given c\in B_k,\, b\in C}
  \end{equation}
  \end{sublemma}
  \begin{proof}
  Can easily be checked.
  \end{proof}
  
  So, by axiom \ref{ax:sets:addpair} and axiom \ref{ax:sets:supersets},  $B^*_{k+1}\in\K$, completing the inductive step.
  
  This has shown by induction that $B^*_K\in\K$. $B^*_K= B^*$, so $B^*\in\K$. 
  
  We need to show that $B\in \K$. This works because it is obtained from $B^*$ simply by multiplying everything by a scalar, so follows from the $n=1$ version of axiom \ref{ax:sets:add}, which itself directly follows from axiom \ref{ax:sets:addpair}. 
  
  This shows that the statement holds for $n+1$, completing the inductive step. Thus we have shown that it holds for all $n$; as required.
  \end{proof attempt}

\end{suggestion}

\end{document}